\documentclass{article}
\usepackage{ijcai19}

\usepackage{times}
\usepackage{soul}
\usepackage{url}
\usepackage[hidelinks,breaklinks=true]{hyperref}
\usepackage[small]{caption}
\usepackage{graphicx}
\usepackage{amsmath}
\usepackage{booktabs}
\usepackage{algorithm}
\usepackage{algorithmic}
\usepackage{xcolor}
\usepackage{amssymb,amsthm}
\usepackage{enumerate}
\usepackage{xspace}

\title{Guarantees for Sound Abstractions for Generalized Planning \\ (Extended Paper)}

\author{%
  Blai Bonet$^{1,2}$\footnote{On sabbatical leave from Universidad Sim\'on Bol\'{\i}var.}\and
  Raquel Fuentetaja$^2$\and
  Yolanda E-Mart\'{\i}n$^2$\and
  Daniel Borrajo$^2$\\
  \affiliations
  $^1$Universidad Sim\'on Bol\'{\i}var, Venezuela\\
  $^2$Universidad Carlos III de Madrid, Spain\\
  \emails
  bonet@usb.ve,
  \{rfuentet,yescuder\}@inf.uc3m.es,
  dborrajo@ia.uc3m.es
}

\newcommand{\Omit}[1]{}
\newcommand{\tup}[1]{\langle #1 \rangle}
\newcommand{\bracket}[1]{\ensuremath{[\![ #1 ]\!]}}
\newcommand{\abst}[2]{\tup{#1 ; #2}}
\newcommand{\inc}[1]{#1\!\!\uparrow}
\newcommand{\dec}[1]{#1\!\!\downarrow}
\newcommand{\complies}[2]{#1 \sim #2}
\newcommand{\repr}[3]{#1 \sim_{#3} #2}

\newcommand{\Q}{\mathcal{Q}}
\newcommand{\D}{\mathcal{D}}
\renewcommand{\L}{\mathcal{L}}
\newcommand{\U}{\mathcal{U}}
\newcommand{\F}{\mathcal{F}}

\newcommand{\G}{\mathcal{G}}
\newcommand{\T}{\mathcal{T}}
\newcommand{\B}{\mathcal{B}}
\newcommand{\Pre}{\text{Pre}}
\newcommand{\Eff}{\text{Eff}}
\newcommand{\Post}{\text{Post}}
\newcommand{\constant}[1]{\ensuremath{\mathsf{#1}}}

\newcommand{\Move}{\text{Move}}
\newcommand{\Newtower}{\text{Newtower}}
\newcommand{\Pick}{\text{Pick}}
\newcommand{\Drop}{\text{Drop}}
\newcommand{\Link}{\text{Link}}
\newcommand{\clear}{{clear}}
\newcommand{\on}{{on}}
\newcommand{\ontable}{{ontable}}
\newcommand{\at}{{at}}
\newcommand{\In}{{in}}
\newcommand{\carry}{{ca}}
\newcommand{\free}{{fr}}

\newcommand{\citeay}[1]{\citeauthor{#1}\xspace [\citeyear{#1}]\xspace}
\providecommand{\noopsort}[1]{}

\newtheorem{definition}{Definition}
\newtheorem{theorem}[definition]{Theorem}
\newtheorem{lemma}[definition]{Lemma}
\newtheorem{corollary}[definition]{Corollary}
\newenvironment{example}{\noindent\textbf{Example.}}{\qed}

\definecolor{darkgreen}{rgb}{0,0.5,0}
\newcommand{\raquel}[1]{\textcolor{red}{Raquel. #1 }}

\begin{document}
\allowdisplaybreaks
\maketitle

\begin{abstract}
  Generalized planning is about finding plans that solve collections
  of planning instances, often infinite collections, rather than
  single instances. Recently it has been shown how to reduce the
  planning problem for generalized planning to the planning problem 
  for a qualitative numerical problem; the latter being
  a reformulation that simultaneously captures 
  all the instances in the collection. An important thread of research
  thus consists in finding such reformulations, or abstractions,
  automatically. A recent proposal learns the abstractions
  inductively from a finite and small sample of transitions from
  instances in the collection. However, as in all inductive processes,
  the learned abstraction is not guaranteed to be correct for the
  whole collection. In this work we address this limitation by
  performing an analysis of the abstraction with respect to the
  collection, and show how to obtain formal guarantees for
  generalization. These guarantees, in the form of first-order
  formulas, may be used to 1)~define subcollections of instances
  on which the abstraction is guaranteed to be sound, 2)~obtain
  necessary conditions for generalization under certain assumptions,
  and 3)~do automated synthesis of complex invariants for planning
  problems. Our framework is general, it can be extended or combined
  with other approaches, and it has applications that go beyond
  generalized planning.
\end{abstract}

\section{Introduction}

Generalized planning is about finding plans that solve a whole collection
of instances of planning problems rather than finding a plan for a single
instance as in classical planning
\cite{srivastava08learning,hu:generalized,srivastava:aaai2011,BelleL16,anders:generalized}.
In its simplest form, the instances in the collection share a common pool
of actions and observable features \cite{hu:generalized,bonet:ijcai2017},
yet other formulations consider relational domains where the actions and
features in the instances result of grounding a collection of actions and
atom schemas with different sets of objects
\cite{boutilier2001symbolic,wang2008first,srivastava:generalized,van2012solving}.

A recent proposal for handling relational domains casts the problem of
generalized planning as the problem of solving a single abstraction, or
reformulation, that captures all the instances in the collection \cite{bonet:ijcai2018}.
This abstraction however involves qualitative numerical features, in
addition to the standard boolean features, that are defined in terms of
the objects in the states and their relationships.
The actions in the abstraction tell how the features change their values
when actions are applied. Qualitative rather than exact numerical
features are used to avoid undecidability issues
\cite{helmert:numerical}. The change for such features is only
\emph{qualitative} as they only specify whether the numerical
feature increases, decreases, or remain unchanged.
Under such effects, the problem of solving the abstraction, and hence
the generalized planning problem, can be reduced to the problem of
solving a single fully observable \emph{non-deterministic} (FOND)
problem \cite{geffner:book}.

This formulation of generalized planning is appealing as it leverages
the existing FOND planners to solve, in one shot, a complete (often infinite)
class of problems, but it requires the right set of features and the right abstraction. 
\citeay{bonet:aaai2019} learn the abstraction
\emph{inductively} from a small sample of transitions from instances
in the collection. 
The abstraction is guaranteed to generalize 
when the sample is sufficiently general and diverse, but, as far as we
know, there have been no attempts to automatically check whether the
learned abstraction is sound for the collection. 

In this work we bridge this gap by providing a general framework
for the synthesis of guarantees for generalization. The guarantees
are in the form of first-order formulas that provide sufficient
conditions for generalization: every instance whose reachable states
satisfy the formulas is guaranteed to be handled correctly by the
abstraction.
We only address the synthesis of such formulas
and defer to future work the problem of verifying whether the formulas
are satisfied on the reachable states of a given instance.
Nonetheless, the automatically synthesized formulas have a rich and complex
structure, and they often express novel and interesting invariants
on well-known benchmarks. For example, in Blocksworld,
the classical problem of moving blocks with a gripper,
one such formula says that every tower must end in a clear block, a formula
that thus forbids the existence of ``circular towers''; we are not aware
of any other approach for invariant synthesis that is able to produce
such a formula.

Our contributions are the following:
1)~a crisp theoretical foundation for the synthesis of formulas only
using as input the relational planning domain and the abstraction,
2)~the obtained formulas define subcollections of instances that are
guaranteed to be handled correctly by the abstraction,
3)~under additional assumptions, necessary conditions for generalization
are obtained, and
4)~the synthesis also provides candidates for invariants
that would then need to be verified. 

The paper is organized as follows. The next section provides
background on the feature-based account for generalized planning.
First-order structures and abstractions are discussed in Sect.\ 3.
The framework for generalization and the synthesis algorithm
are given in Sect.\ 4 and 5.
Sect.\ 6, discusses necessary conditions 
and the synthesis of invariants.
The paper concludes with examples and a discussion. 

\section{Background}

\subsection{Collections of Instances}

We consider collections $\Q$ of \emph{grounded} STRIPS 
instances $P=(F,A,I,G)$ where $F$ is a set of atoms
(propositions), $A$ is a set of actions, and $I\subseteq F$
and $G\subseteq F$ describe the initial and goal states of $P$.
It is assumed that all instances in $\Q$ result from grounding a
\emph{common domain} $\D$ with a set of objects, particular to
each instance, and descriptions of the initial and goal situations.
As it is standard, $\D$ specifies the constant and predicate
symbols that define the propositions via the grounding
process, and it also contains lifted action schemas that generate
the set $A$ of grounded actions.
$\Q(\D)$ denotes the class of all grounded instances for domain $\D$.
Hence, $\Q\subseteq\Q(\D)$ as all instances in $\Q$ come from $\D$.

\subsection{Abstractions}

The boolean and numerical features are used to build
\emph{uniform abstractions} for the instances in $\Q$.
Such instances, although sharing a common relational domain,
may differ substantially in the number of actions, objects,
and observables.

A boolean feature $\phi$ for $\Q$ is a function that maps each
instance $P\in\Q$ and state $s$ for $P$ (reachable from the initial
state of $P$) into a truth value $\phi(P,s)\in\{\top,\bot\}$.
A numerical feature is a function $\phi$ that maps $P$ and $s$
into a non-negative integer $\phi(P,s)$. When $P$ or $s$ 
are clear from context we may simplify notation.
The set of features for $\Q$ is denoted by $F$.
For boolean features $f$, an $F$-literal is either $f$ or $\neg f$,
while for numerical features $n$, an $F$-literal is $n>0$ or $n=0$.

An abstraction for $\Q$ is a tuple $\bar Q=(F,A_F,I_F,G_F)$ where
$F$ is a set of features, $A_F$ is a set of abstract actions, and
$I_F$ and $G_F$ describe the abstract initial and goal states in terms of
the features. 
An abstract action $\bar a$ is a pair $\abst{\Pre}{\Eff}$ where
$\Pre$ is a collection of $F$-literals, and $\Eff$ is a collection
of effects for $F$.
Effects for boolean features are denoted by $F$-literals, while
effects for numerical features $n$ correspond to increments or
decrements denoted by $\inc{n}$ or $\dec{n}$ respectively.
The items $I_F$ and $G_F$ denote \emph{consistent}
sets of $F$-literals. 
It is assumed that the effects of actions and $G_F$ are consistent
sets of literals, and that $I_F$ is maximal consistent.%
\footnote{A set $S$ of $F$-literals is consistent if for any
  boolean feature $f$, $S$ excludes either $f$ or $\neg f$, and
  for any numerical feature $n$, $S$ excludes either $n>0$ or $n=0$.
  $S$ is maximal consistent if it is consistent, and $S\cup\{L\}$
  is not consistent for any $F$-literal $L\notin S$.
}

The pair $(I_F,G_F)$ of initial and goal states in the abstraction
$\bar Q=(F,A_F,I_F,G_F)$ complies with $\Q$ when $(I_F,G_F)$ complies
with each instance $P$ in $\Q$.
The pair $(I_F,G_F)$ complies with the instance $P$ when the initial state of $P$
is consistent with $I_F$, and if $s$ is a state in $P$ that is 
consistent with $G_F$, then $s$ is a goal state for $P$.
A state $s$ in $P$ is consistent with $I_F$ (resp.\ $G_F$)
iff $\bar s\cup I_F$ (resp.\ $\bar s\cup G_F$) is consistent, where
$\bar s$ denotes the \emph{boolean valuation} of $F$ on $s$; i.e.,
$\bar s=\{ f : f(s){=}\top \} \cup \{ \neg f : f(s){=}\bot \} \cup \{ n{=}0 : n(s){=}0 \} \cup \{ n{>}0 : n(s){>}0 \}$.
If the pair $(I_F,G_F)$ for the abstraction $\bar Q$ complies with $P$,
we write $\complies{\bar Q}{P}$.

Following \citeay{bonet:ijcai2018}, an abstraction $\bar Q$ is sound
for $\Q$ if it complies with $\Q$ and each action $\bar a$
in $A_F$ is sound (for $\Q$).
An abstract action $\bar a=\abst{\Pre}{\Eff}$ is sound iff for
each instance $P$ in $\Q$ and reachable state $s$ in $P$ where
$\Pre$ holds in $\bar s$, $\bar a$ represents at least one action $a$ from
$P$ in $s$.
The abstract action $\bar a$ 
represents the action $a$ in the state $s$ iff
1)~the preconditions of $a$ and $\bar a$ both hold in $s$ and $\bar s$ respectively, and
2)~the effects of $a$ and $\bar a$ over $F$ are similar; namely,
\begin{enumerate}[$a)$]
  \item for any boolean feature $f$ in $F$, if $f$ changes
    from true to false (resp.\ false to true) in the
    transition $s\leadsto res(a,s)$ (where $res(a,s)$ is the
    state that results of applying $a$ in $s$), then
    $\neg f\in\Eff$ (resp.\ $f\in\Eff$),
  \item for any boolean feature $f$ in $F$, if $f$ (resp.\
    $\neg f$) is in $\Eff$, then $f$ is true (resp.\ false)
    in $res(a,s)$, and
  \item for each numerical feature $n$ in $F$, $\dec{n}$
    (resp.\ $\inc{n}$) appears in $\Eff$ if and only if $n(P,res(a,s)) < n(P,s)$
    (resp.\ $n(P,s) < n(P,res(a,s))$).
\end{enumerate}
We write $\repr{\bar a}{a}{P,s}$ to denote that the abstract
action $\bar a$ represents the action $a$ in the (reachable)
state $s$ of $P$. In such a case, we also say that $a$ instantiates
$\bar a$ in $s$.
When there is no confusion about $P$, we simplify notation
to $\repr{\bar a}{a}{s}$. 

Soundness links plans for $\bar Q$ with generalized plans:
if $\bar\pi$ is a plan that solves an abstraction $\bar Q$ that is \emph{sound} for
$\Q$ and $P$ is an instance in $\Q$, then \emph{any} execution $(a_0,a_1,\ldots)$
spawned by $\bar\pi$ on $P$ reaches a goal state for $P$.
The execution $(a_0,a_1,\ldots)$ is spawned by $\bar\pi$ on $P$ iff 
1)~$a_i$ instantiates $\bar\pi(\bar s_i)$ in $s_i$, for $i\geq0$,
2)~$\bar s_i$ is the boolean valuation of $s_i$, for $i\geq0$,
3)~$s_{i+1}=res(a_i,s_i)$, for $i\geq0$, and
4)~$s_0$ is the initial state of $P$.

\medskip
\begin{example}
  Consider the collection $\Q_{clear}$ with
  all Blocks{-}world instances with goal $\clear(\constant{A})$
  where $\constant{A}$ is a fixed block.
  The domain $\D_{clear}$ has no explicit gripper, contains
  a single constant $\constant{A}$, and has two
  action schemas:
  $\Newtower(x,y)$ to move block $x$ from block $y$ to the table, and
  $\Move(x,y,z)$ to move block $x$ from block $y$ onto block $z$.
  An abstraction for $\Q_{clear}$ is $\bar Q_{clear}=(F,A_F,I_F,G_F)$
  where $F=\{n\}$ and $n$ is the feature that
  counts the number of blocks above $\constant{A}$,
  $A_F=\{\bar a\}$ where 
  $\bar a=\abst{n{>}0}{\dec{n}}$, $I_F=\{n{>}0\}$, and
  $G_F=\{n{=}0\}$.
  It is easy to check that $\bar Q_{clear}$ is sound
  and solved by the plan $\bar\pi_{clear}$ that executes
  $\bar a$ whenever $n{>}0$.
  An action $\Newtower(x,y)$ or $\Move(x,y,z)$ that
  ``removes'' a block from above $\constant{A}$ in state $s$
  is an action that instantiates $\bar a$ in $s$.
  Notice that $\Q_{clear} \neq \Q(\D_{clear})$ since, for example,
  $\Q(\D_{clear})$ 
  contains instances that have ``circular towers''.
\end{example}

\subsection{Inductive Learning and Concepts}

\citeay{bonet:aaai2019} show how an abstraction can be learned
from a sample of transitions and a collection $\F$ of candidate
features.
In their approach, each feature in $\F$ is associated with a
concept $C$ that is obtained from a set of atomic concepts,
and a concept grammar \cite{dl-handbook}.%
\footnote{We do not consider the distance features \cite{bonet:aaai2019}
  as it is not clear how to express them in first-order logic.
}

In general, a concept for $\Q$ may be thought of as a function $C$ that
maps instances $P$ in $\Q$ and states $s$ in $P$ into sets $C(P,s)$
of \emph{tuples of objects}.
Concepts define features: boolean features $f_C$ that
denote whether $C(P,s)$ is non-empty, and numerical
features $n_C$ that denote the cardinality of $C(P,s)$.
The concepts by \citeauthor{bonet:aaai2019}\xspace are limited
to denotations that are subsets of objects
rather than object tuples.

\section{First-Order Abstractions}

We deal with formulas in first-order logic 
that are built from a signature $\sigma=\sigma(\D)$
given by the relational domain $\D$.
The constants defined in $\D$ appear as constant symbols
in $\sigma$, and the predicates defined in $\D$ appear
as relational symbols of corresponding arity in $\D$.
The signature also contains binary relations $p^+$ and $p^*$
for the binary predicates $p$ in $\D$. 
As usual, $\L(\sigma)$ denotes the class of well-formed 
formulas over $\sigma$.

First-order formulas are interpreted over first-order structures, also called interpretations.
We are only interested in structures 
that are associated with states. 
A state $s$ provides the universe $\U_s$ of objects
and the interpretations for the constant and relational symbols in $\D$.
The interpretations for $p^+$ and $p^*$, for the binary predicates $p$, 
are provided by the transitive and reflexive-transitive closure of the
interpretation of $p$ provided by $s$.
We write $\varphi(\bar x)$ to denote a formula whose free
variables are among those in $\bar x$.
If $\varphi(\bar x)$ is a formula, $s$ is a state in $P$, 
and $\bar u$ is a tuple of objects in $\U_s$ of
dimension $|\bar x|$, $s\vDash\varphi(\bar u)$ denotes that the
interpretation provided by $s$ satisfies $\varphi$ when the variables
in $\bar x$ are interpreted by the corresponding objects in $\bar u$.

For a concept $C$ characterized by $\Psi_C(\bar x)$,
the extension of $C$ for $s$ in $P$ is 
$C^s=C(P,s)=\{\bar u:s\vDash\Psi_C(\bar u)\}$.
We assume that all features correspond to concepts whose
characteristic functions are first-order definable:

\begin{definition}[First-Order Abstraction]
  Let $\D$ be a planning domain and let $\sigma=\sigma(\D)$ be
  the signature for $\D$.
  A concept $C$ is (first-order) $\D$-definable if $\Psi_C(\bar x)$
  belongs to $\L(\sigma)$.
  A feature $f$ is $\D$-definable if $f$ is given by a concept
  $C$ that is $\D$-definable.
  An abstraction $\bar Q=(F,A_F,I_F,G_F)$ is a first-order
  abstraction for $\D$ if each feature $f$ in $F$ is $\D$-definable.
\end{definition}

When $\D$ is clear from the context, we just say that $\bar Q$ is
a first-order abstraction without mentioning $\D$.
The applicability of an abstract action $\bar a$ in a
first-order abstraction $\bar Q$ on a state $s$ can be decided
with a first-order formula $\Pre(\bar a)$. 

\medskip
\begin{example}
  The abstraction $\bar Q_{clear}$ is a first-order abstraction because
  $F=\{n\}$ and $n$ is the cardinality of the
  concept $C$ given by $\Psi_C(x)=\exists y(\on(x,y) \land \on^*(y,\constant{A}))$.
  However, $C$ is also given by $\Psi'_C(x)=\on^+(x,\constant{A})$.
  As usual, both representations may yield different results
  although being \emph{logically equivalent}; more about this below.
\end{example}

\section{Conditions for Generalization}

Let $\bar Q=(F,A_F,I_F,Q_F)$ be a first-order abstraction for
$\D$.
We look for conditions to establish the soundness of
$\bar Q$ for a generalized problem $\Q\subseteq\Q(\D)$.
In particular, we aim for conditions of the form
$\G=\{ \Phi_{\bar a} : \bar a \in A_F \}$
where $\Phi_{\bar a}=\exists\bar z\bigl(\bigvee_i \Psi^{a_i}_{\bar a}(\bar z)\bigr)$
is associated with the abstract action $\bar a$ and satisfies the following: 
\begin{enumerate}[--]
  \item $a_i$ is an \emph{action schema} in $\D$,
  \item $\bar z$ is a tuple of variables that represent the parameters
    of the action schemas in $\D$ (these are existentially quantified
    on the objects of the given state $s$ in problem $P$), and
  \item if $\bar o$ is a tuple of objects of dimension $|\bar z|$ such
    that $s\vDash\Pre(\bar a)\land\Psi^{a_i}_{\bar a}(\bar o)$, where
    $s$ is a reachable state in problem $P\in\Q(\D)$, then the
    \emph{ground action} $a_i(\bar o)$ instantiates the abstract action
    $\bar a$ in the state $s$ (i.e., $\repr{\bar a}{a_i(\bar o)}{s}$).
\end{enumerate}
The idea is that $\Psi^{a_i}_{\bar a}(\bar o)$ suffices to establish
$\repr{\bar a}{a_i(\bar o)}{s}$ directly from $\Pre(\bar a)$ and the
\emph{(lifted) domain} $\D$ without using any other information about
the reachability of state $s$ (e.g., invariant information for reachable
states).
On the other hand, such formulas would be ``accompanied'' by
\emph{assumed conditions}
$\Pre(\bar a)\Rightarrow \exists\bar z\bigl(\bigvee_i \Psi^{a_i}_{\bar a}(\bar z)\bigr)$
on the \emph{reachable states} that together with the above
properties provide the guarantee:

\begin{definition}[Guarantee]
  \label{def:guarantee}
  Let $\D$ be a planning domain and let $\bar Q=(F,A_F,I_F,G_F)$
  be a first-order abstraction. 
  A \emph{guarantee} for $\bar Q$ is a set 
  $\G=\{\Phi_{\bar a}=\exists\bar z\bigl(\bigvee_i\Psi^{a_i}_{\bar a}(\bar z)\bigr):\bar a \in A_F\}$
  of formulas for each abstract action $\bar a$ in $\bar Q$.
  The guarantee is valid in instance $P\in\Q(\D)$ iff for each state
  $s\in P$ (reachable or not) and tuple $\bar o$ of objects in $P$,
  if $s\vDash\Pre(\bar a) \land \Psi^{a_i}_{\bar a}(\bar o)$ then
  $\repr{\bar a}{a_i(\bar o)}{s}$.
  The guarantee $\G$ is valid for $\D$ iff it is valid for each problem $P$
  in $\Q(\D)$.
\end{definition}

\begin{theorem}[Soundness]
  \label{thm:guarantee:soundness}
  Let $\D$ be a planning domain, let $\bar Q=(F,A_F,I_F,I_F,G_F)$
  be a first-order abstraction, and let 
  $\G=\{\Phi_{\bar a}=\exists\bar z\bigl(\bigvee_i\Psi^{a_i}_{\bar a}(\bar z)\bigr):\bar a \in A_F\}$
  be a guarantee for $\D$.
  If $\G$ is valid, then $\bar Q$ is a sound abstraction for the
  generalized problem
  $\Q=\{P \in \Q(\D) : \text{$\complies{\bar Q}{P}$ and $\Pre(\bar a)\Rightarrow\Phi_{\bar a}$}$
  holds in the reachable states in $P\}$.
\end{theorem}
\begin{proof}
  Let $P$ be a problem in $\Q$, let $s$ be a reachable state in $P$,
  and let $\bar a$ be an abstract action that is applicable in $\bar s$.
  Since $\complies{\bar Q}{P}$, we only need to show $\repr{\bar a}{a}{s}$
  for some action $a$.
  By definition of $\Q$, $s\vDash\Phi_{\bar a}$ where
  $\Phi_{\bar a}=\exists\bar z\bigl(\bigvee_i\Psi^{a_i}_{\bar a}(\bar z)\bigr)$.
  Hence, there is a tuple $\bar o$ of objects such that
  $s\vDash\Pre(\bar a)\land\Psi^{a_i}_{\bar a}(\bar o)$ for some schema $a_i$ in $\D$.
  Then, by Definition~\ref{def:guarantee}, $\repr{\bar a}{a_i(\bar o)}{s}$.
\end{proof}

\section{Synthesis}

For a feature $f$ defined by concept $C$ we need to track
its value along transitions $s\leadsto res(a(\bar o),s)$.
Let $\Psi^a_C(\bar z,\bar x)$ be a formula that defines
at state $s$ the extension of $C$ in the state
$res(a(\bar o),s)$ that results of applying $a(\bar o)$ in $s$; i.e.,
\begin{alignat*}{1}
  s\vDash\Psi^a_C(\bar o,\bar u) \quad\text{iff}\quad \bar u\in C^{res(a(\bar o),s)}
\end{alignat*}
where $\bar u$ is a tuple of objects.
For example, a boolean feature $f$ defined by $C$ goes from true to false
in $s\leadsto res(a(\bar o),s)$ iff $C^s\neq\emptyset$
and $C^{res(a(\bar o),s)}=\emptyset$ iff $s\vDash\exists\bar x(\Psi_C(\bar x))$
and $s'\vDash\neg\exists\bar x(\Psi_C(\bar x))$ iff
$s\vDash \exists\bar x(\Psi_C(\bar x)) \land \forall\bar x(\neg\Psi_C^a(\bar o,\bar x))$.

Since the concept $C$ may be defined in terms of relations
$p^*$ or $p^+$ that denote the transitive closure of $p$,
and that transitive closure is not first-order definable
\cite{vardi:complexity-ql}, it is not always possible to
track in first-order logic the change of denotation for
$p^*$ or $p^+$ after an action changes the denotation of $p$.
Hence, we settle for a ``logical approximation'' of
$\Psi^a_C(\bar z,\bar x)$ in terms of necessary and sufficient
conditions:
\begin{alignat*}{1}
  S^a_C(\bar z,\bar x)\ \Rightarrow\
    \Psi^a_C(\bar z,\bar x)\ \Rightarrow\ N^a_C(\bar z,\bar x) \,.
\end{alignat*}
A \emph{base for synthesis} provides approximations
for all the atoms in the language $\L(\sigma)$:

\begin{definition}[Base for Synthesis]
  \label{def:base}
  A base for synthesis for domain $\D$ is a set $\T$ that contains
  formulas $\T_X(a,p)(\bar z,\bar x)$ for $X{\in}\{N,S\}$, action
  schemas $a(\bar z)\in\D$, and predicates $p\in\D$ of arity $|\bar x|$.
  It also contains formulas $\T_X(a,p^c)(\bar z,x,y)$ for $X{\in}\{N,S\}$,
  action schemas $a(\bar z)\in\D$, binary predicates $p\in\D$, and at
  least one of $c=*$ or $c=+$.
  These formulas should provide necessary and sufficient conditions
  as follows.
  For any problem $P$ in $\Q(\D)$, state $s$ in $P$, tuple $\bar o$
  such that $a(\bar o)$ is applicable at $s$, tuple $\bar u$,
  objects $u$ and $v$, and $c\in\{*,+\}$:
  \begin{alignat*}{1}
    \notag s\vDash\T_S(a,p)(\bar o,\bar u)\
      &\Rightarrow\ \bar u \in C^{res(a(\bar o),s)}_p \\
      &\Rightarrow\ s\vDash\T_N(a,p)(\bar o,\bar u)\,, \\
    \notag s\vDash\T_S(a,p^c)(\bar o,u,v)\
      &\Rightarrow\ (u,v)\in C^{res(a(\bar o),s)}_{p^c} \\
      &\Rightarrow\ s\vDash\T_N(a,p^c)(\bar o,u,v)
  \end{alignat*}
  where $C^s_p\!=\!\{\bar u:s\vDash p(\bar u)\}$ and
  $C^s_{p^c}\!=\!\{(u,v):s\vDash p^c(u,v)\}$ are the
  concepts associated with $p$ and $p^c$ respectively.
\end{definition}

The approximation for the atoms in $\L(\sigma)$ that is provided
by the base $\T$ is \emph{lifted} over all first-order formulas.
Indeed, the following structural induction gives necessary and
sufficient conditions $N^a_\varphi(\bar z,\bar x)$ and
$S^a_\varphi(\bar z,\bar x)$ for any concept $C$ defined
in terms of formula $\varphi$. For $X{\in}\{N,S\}$:
\begin{enumerate}[--]
  \item $X^a_{p}(\bar z,\bar x){=}\T_X(a,p)(\bar z,\bar x)$ 
    where $p$ is a predicate of arity $|\bar x|$, or $p=q^c$ for
    some binary predicate $q$, and $c\in\{*,+\}$,
  \item $X^a_{\varphi\circ\psi}(\bar z,\bar x)=X^a_\varphi(\bar z,\bar x) \circ X^a_\psi(\bar z,\bar x)$
    where $\circ\in\{\land,\lor\}$,
  \item $N^a_{\neg\varphi}(\bar z,\bar x)=\neg S^a_\varphi(\bar z,\bar x)$ and $S^a_{\neg\varphi}(\bar z,\bar x)=\neg N^a_\varphi(\bar z,\bar x)$, and
  \item $X^a_{Qy\varphi}(\bar z,\bar x)=Qy X^a_\varphi(\bar z,\bar x,y)$
    where $Q\in\{\exists,\forall\}$.
\end{enumerate}
The base provides approximations for either $p^*$ or $p^+$,
or both. In the former case, this is enough since one of
the closures can be expressed in terms of the other; e.g.,
$\on^+(x,y) \equiv \exists z(\on(x,z)\land \on^*(z,y))$).

Below we propose a general base for synthesis of formulas.
With this base, the formulas $N^a_\varphi$
and $S^a_\varphi$ are identical except when $\varphi$
contains a transitive closure.
Hence, except for such $\varphi$, both formulas
are necessary and sufficient.

\begin{theorem}[Lift]
  \label{thm:lift}
  Let $\T$ be a base for synthesis for domain $\D$, let $a=a(\bar z)$ be an
  schema in $\D$, and let $\varphi(\bar z, \bar x)$ be a first-order formula
  in $\L(\sigma(\D))$.
  Then, for any instance $P$ in $\Q(\D)$, state $s$ for $P$,
  and tuples $\bar o$ and $\bar u$ of objects in $P$:
  \begin{alignat*}{1}
    s \vDash S^a_\varphi(\bar o, \bar u)\ \Rightarrow\
    res(a(\bar o),s) \vDash \varphi(\bar o, \bar u)\ \Rightarrow\
    s \vDash N^a_\varphi(\bar o, \bar u) \,.
  \end{alignat*}
\end{theorem}

As noted earlier, tracking the change of boolean features $f$
defined by concepts $C$ is easy since $f$ is true or false
at $s$ iff $C^s$ is non-empty or empty respectively.
Tracking the qualitative numerical changes is more challenging, however.
For example, $f$ increases in the transition $s\leadsto s'$
iff $|C^s|<|C^{s'}|$. This condition is difficult to
capture because the extension of $C$ may increase size by
the result of a small change, as simple as one new element
entering the set, or by a large change involving many elements.
The case of local, small, changes is common and easy to define:

\begin{definition}[Monotonicity]
  \label{def:monotoinicity}
  Let $\D$ be a domain
  and let $P$ be an instance in $\Q(\D)$.
  A concept $C$ for $\Q(\D)$ is \emph{monotone} in $P$ if for every \emph{reachable state}
  $s$ in $P$, and action $a(\bar o)$ that is applicable in $s$, either
  $C(P,s)\subseteq C(P,s')$, $C(P,s)\supseteq C(P,s')$, or
  $C(P,s)=C(P,s')$ for $s'{=}res(a(\bar o),s)$.
  A first-order abstraction $\bar Q$ is monotone for $P$ if
  each feature $f$ in $\bar Q$ is defined by a concept $C$ that
  is $\D$-definable and monotone. 
\end{definition}

Necessary and sufficient conditions for the change of value of
monotone features $f$ along transitions $s\leadsto res(a(\bar z),s)$,
for action schema $a(\bar z)$, are provided by the formulas:
\begin{alignat*}{1}
  X^{inc}_C(\bar z)\ &=\ \exists\bar x(\neg\Psi \land X^a_\Psi) \,, \\
  X^{dec}_C(\bar z)\ &=\ \exists\bar x(\Psi \land \neg\widehat X^a_\Psi) \,, \\
  X^{eq}_C(\bar z)\ &=\ \forall\bar x(\Psi \Rightarrow X^a_\Psi) \land \forall\bar x(\widehat X^a_\Psi \Rightarrow \Psi) \,, \\
  X^{true}_C(\bar z)\ &=\ \exists\bar x(X^a_{\Psi}(\bar z,\bar x)) \,, \\
  X^{false}_C(\bar z)\ &=\ \forall\bar x(\neg\widehat X^a_{\Psi}(\bar z,\bar x))
\end{alignat*}
where $C$ is the concept that defines $f$, $X{\in}\{N,S\}$
denotes a necessary or sufficient condition, $\Psi$ and $X^a_\Psi$
denote $\Psi_C(\bar x)$ and $X^a_{\Psi_C}(\bar z,\bar x)$ respectively,
and $\widehat N^a_\Psi=S^a_\Psi$ and $\widehat S^a_\Psi = N^a_\Psi$.

For example, $S^{dec}_C(\bar z)=\exists\bar x(\Psi\land\neg N^a_\Psi)$.
If $s\vDash S^{dec}_C(\bar o)$, there is object tuple $\bar u$ such that
$s\vDash \Psi(\bar u) \land \neg N^a_\Psi(\bar o, \bar u)$; i.e.,
$\bar u\in C^s$ and $\bar u\notin C^{res(a(\bar o),s)}$.
For monotone features, the only possibility is
$C^{res(a(\bar o),s)} \subsetneq C^s$ which means that the feature
\emph{decreases} in the transition $s\leadsto res(a(\bar o),s)$.

For obtaining sufficient conditions for general features, the
first two formulas from above are strengthen as
\begin{alignat*}{1}
  S^{inc}_C(\bar z)\ &=\ \forall\bar x(\Psi \Rightarrow S^a_\Psi) \land \exists\bar x(\neg\Psi \land S^a_\Psi) \,, \\
  S^{dec}_C(\bar z)\ &=\ \forall\bar x(N^a_\Psi \Rightarrow \Psi) \land \exists\bar x(\Psi \land \neg N^a_\Psi) 
\end{alignat*}
where the added conjunct enforces that the feature defined by
the concept $C$ is indeed monotone. For the remaining cases,
the formulas for sufficiency correspond to those above.

\begin{lemma}
  \label{lemma:sufficient}
  Let $C$ be a concept characterized by formula $\Psi_C(\bar x)$,
  and let $s$ be a state in $P\in\Q(\D)$ on which the action $a(\bar o)$
  is applicable. Then, 
  \begin{alignat*}{2}
    s\vDash S^{inc}_C(\bar o)\  &\Rightarrow\ C^s \subsetneq C^{s'}\ &&\Rightarrow\ |C^s|<|C^{s'}| \,, \\
    s\vDash S^{dec}_C(\bar o)\  &\Rightarrow\ C^s \supsetneq C^{s'}\ &&\Rightarrow\ |C^s|>|C^{s'}| \,, \\
    s\vDash S^{eq}_C(\bar o)\   &\Rightarrow\ C^s = C^{s'}\          &&\Rightarrow\ |C^s|=|C^{s'}| \,, \\
    s\vDash S^{true}_C(\bar o)\ &\Rightarrow\ C^{s'}\neq\emptyset\   &&\Rightarrow\ |C^{s'}|>0 \,, \\
    s\vDash S^{false}_C(\bar o)\ &\Rightarrow\ C^{s'}=\emptyset\      &&\Rightarrow\ |C^{s'}|=0
  \end{alignat*}
  where $s'=res(a(\bar o),s)$ is the result of applying $a(o)$ in $s$.
\end{lemma}

\begin{table*}[t]
  \centering
  \begin{tabular}{lcc}
    \toprule
    Reference & Formula \\
    \midrule
    $\B_X(a,p)(\bar z,\bar x)$   & $\bracket{p(\bar x)\in\Post} \lor (p(\bar x) \land \bracket{\neg p(\bar x)\notin\Post})$ \\
    \midrule
    $\B_N(a,p^*)(\bar z,x,y)$    & $p^*(x,y) \lor \exists uv \bigl( \bracket{p(u,v)\in\Post} \land   p^*(x,u) \land p^*(v,y)  \bigr)$ & (action adds at most 1 $p$-atom)  \\[.25em]
                                 & $p^*(x,y) \lor \exists uv \bigl( \bracket{p(u,v)\in\Post} \land ( p^*(x,u) \lor p^*(v,y) ) \bigr)$ & (action adds 2 or more $p$-atoms) \\[.25em]
    $\B_S(a,p^*)(\bar z,x,y)$    & $(x=y) \lor \bigl( p^*(x,y) \land \forall uv (\bracket{\neg p(u,v)\in\Post} \Rightarrow u=v) \bigr)$ \\
    \bottomrule
  \end{tabular}
  \caption{General base $\B$ for synthesis of any domain $\D$. $\Post(a(\bar z))$ is abbreviated by $\Post$. $X\in\{N,S\}$.
    There are two versions of the necessary condition for $p^*$; one for actions that add at most one atom $p(u,v)$,
    and the other for actions that add two or more atoms of this
    form. The first version uses a conjunction,
    $p^*(x,u) \land p^*(v,y)$, while the second version replaces it
    with a disjunction.
  }
  \label{table:base}
\end{table*}

We have expressed how the value of individual features
changes in transitions. Before providing the complete synthesis,
we need to express the value of preconditions of abstract actions,
and how the actions affect the different features.

Preconditions of abstract actions $\bar a$ on features $f=f_C$
are expressed by
$\Pre(\bar a)_C=\top$ if there is no precondition on $f$,
$\Pre(\bar a)_C = \exists\bar x(\Psi_C(\bar x))$ if $f$ is boolean
(resp.\ numeric) and $\Pre(\bar a)$ contains $f$ (resp.\ $f>0$), and
$\Pre(\bar a)_C = \forall\bar x(\neg\Psi_C(\bar x))$ if $f$ is boolean
(resp.\ numeric) and $\Pre(\bar a)$ contains $\neg f$ (resp.\ $f=0$).

On the other hand, $\bar a$ partitions the set of features according
to their type and the effects of $\bar a$ on them:
\begin{alignat*}{1}
  \Delta^{inc}_{\bar a}\
    &=\ \{ n \in F : \text{$n$ is numeric and $\inc{n}\,\in\Eff(\bar a)$} \} \,, \\
  \Delta^{dec}_{\bar a}\
    &=\ \{ n \in F : \text{$n$ is numeric and $\dec{n}\,\in\Eff(\bar a)$} \} \,, \\
  \Delta^{eq}_{\bar a}\
    &=\ \{ f \in F : \text{$f$ is not affected by $\bar a$} \} \,, \\
  \Delta^{true}_{\bar a}\
    &=\ \{ f \in F : \text{$f$ is boolean and $f\in\Eff(\bar a)$} \} \,, \\
  \Delta^{false}_{\bar a}\
    &=\ \{ f \in F : \text{$f$ is boolean and $\neg f\in\Eff(\bar a)$} \} \,.
\end{alignat*}

\begin{definition}[Synthesis]
  \label{def:synthesis}
  Let $\T$ be a base for synthesis for domain $\D$, and let
  $\bar Q=(F,A_F,I_F,G_F)$ be a first-order abstraction for $\D$.
  For abstract action $\bar a$ in $A_F$ and schema $a(\bar z)$
  in $\D$, we define the formula $\Psi^a_{\bar a}(\bar z)$ as
  \begin{alignat*}{1}
    \Pre(a(\bar z))\ \land
    \bigwedge_{\varphi\in Chg} 
      \bigwedge_{f_C\in\Delta^\varphi_{\bar a}}
        \Pre(\bar a)_C \land S^\varphi_C(\bar z) 
  \end{alignat*}
  where $Chg{=}\{inc,dec,eq,true,false\}$.
  The guarantee for $\bar Q$ is 
  $\G(\T,\bar Q)=\{\Phi_{\bar a}=\exists \bar z\bigl(\bigvee_{a\in\D} \Psi^{a}_{\bar a}(\bar z)\bigr):\bar a \in A_F\}$.
\end{definition}

\begin{theorem}[Main]
  \label{thm:main}
  Let $\T$ be a base for synthesis for domain $\D$, and let
  $\bar Q=(F,A_F,I_F,G_F)$ be a first-order abstraction. 
  Then, $\G(\T,\bar Q)$ is a \emph{valid guarantee} for $\D$
  (cf.\ Definition~\ref{def:guarantee}).
\end{theorem}

We cannot yet provide a complete example because the synthesis
requires the conditions for the atoms in the language that are
given by the base for synthesis.
We now provide one such base, and apply it to the running example. 

\subsection{A General Base for Synthesis}

The synthesis framework is parametrized by the base.
Trivial, non-informative, bases are easy to obtain:
it is enough to define sufficient and necessary conditions
as $\bot$ and $\top$ respectively for each atom in the language.
We provide a simple, general, and non-trivial base that can be
used with any domain $\D$.
The conditions provided by two different bases, or by the
same base for different but logically equivalent formulas,
do not need to be logically equivalent.

Table~\ref{table:base} shows a template for obtaining bases $\B(\D)$
for any domain $\D$. No formula in the template involves the predicate
$p^+$; i.e., all such predicates have been replaced by equivalent formulas
involving $p^*$. (Alternatively, we may define a base that only
resolves $p^+$ and assumes that no formula contains $p^*$.)
Two versions for the necessary condition for $p^*$ are provided:
one when the action $a$ adds at most one atom $p(u,v)$, and the
other when $a$ adds two or more such atoms.

The formulas in Table~\ref{table:base} involve ``bracket expressions''
that instantiate to first-order formulas.
For schema $a(\bar z)$ and tuple $\bar x$, a bracket
expression reduces to either to a logical constant $\top$ or $\bot$,
or to an expression involving equality over the variables
in $\bar z$ and $\bar x$, and the constant symbols in $\D$.
For example, $\bracket{\neg \on(x,y)\notin\Post}$ reduces
to $xy\neq z_1z_2$ for the action $\Newtower(z_1,z_2)$
since this action removes only $\on(z_1,z_2)$.

\begin{theorem}[General Base]
  \label{thm:base:general}
  Let $\D$ be a planning domain.
  The set $\B(\D)$ is a base for synthesis for domain $\D$.
\end{theorem}

\begin{corollary}
  \label{cor:base:general}
  Let $\D$ be a domain and let $\bar Q=(F,A_F,I_F,$ $G_F)$ be a
  first-order abstraction for $\D$. The guarantee $\G(\B(\D),\bar Q)$ is valid for
  $\D$ and, hence, $\bar Q$ is a \emph{sound abstraction} for the
  generalized problem $\Q=\{P\in\Q(\D): \complies{\bar Q}{P}$ and
  $\Pre(\bar a)\Rightarrow\Phi_{\bar a}$ holds in the reachable
  states in $P\}$.
\end{corollary}

\medskip
\begin{example}
  The abstraction $\bar Q_{clear}\!=\!(F,A_F,I_F,G_F)$ has
  a single feature $n{=}n_C$ for $\Psi_C(x){=}\exists y(\on(x,y) \land  \on^*(y,\constant{A}))$.
  $\D_{clear}$ has two schemas $a_1=\Newtower(z_1,z_2)$ and $a_2=\Move(z_3,$ $z_4,z_5)$.
  The condition $S^{dec}_C(z_1,z_2)$ for $a_1$ is equivalent (after simplification) to
  \begin{alignat*}{1}
    \label{eq:dec:newtower}
    \on(z_1,\!z_2) \land \on^*(z_2,\!\constant{A}) \land \forall y(\on(z_1,\!y)\!\land\!\on^*(y,\!\constant{A}) \Rightarrow y{=}z_2).
  \end{alignat*}
  The formula $\Psi^{a_1}_{\bar a}(z_1,z_2)$ is this formula conjoined with $\clear(z_1)$.
  For action $a_2$, $S^{dec}_C(z_3,z_4,z_5)$ is
  \begin{alignat*}{1}
    &\on^+(z_3,\constant{A}) \land \neg \on^*(z_5,\constant{A})\ \land \\
    &\qquad\qquad\qquad\forall y( \on(z_3,y) \land \on^*(y,\constant{A}) \Rightarrow y=z_4) \,.
  \end{alignat*}
  The formula $\Psi^{a_2}_{\bar a}(z_1,z_2)$ is this formula
  conjoined with $\clear(z_3)$, $\on(z_3,z_4)$, and $\clear(z_5)$.
  The guarantee for $\bar a$ is
  $\Phi_{\bar a}=\exists\bar z\bigl(\Psi^{a_1}_{\bar a}(z_1,z_2) \lor \Psi^{a_2}_{\bar a}(z_3,z_4,z_5)\bigr)$.

  By Corollary~\ref{cor:base:general}, $\Q_{clear}$ is sound for
  instances with goal $\clear(\constant{A})$ and
  \emph{reachable states} that satisfy
  $\exists x(\on^+(x,\constant{A})) \Rightarrow \Phi_{\bar a}$.
  Namely, if there is a block above $\constant{A}$, then either
  there are blocks $z_1$ and $z_2$ such that $z_1$ is clear and
  on $z_2$, $z_2$ is $\constant{A}$ or above it, and $z_2$ mediates
  any ``path of blocks'' from $z_1$ to $\constant{A}$,
  or there are blocks $z_3$, $z_4$ and $z_5$ such that $z_3$ is
  clear, on $z_4$, and above $\constant{A}$, $z_5$ is clear and
  not equal to $\constant{A}$ or above it, and $z_4$ mediates
  any path from $z_3$ to $\constant{A}$.
  This formula indeed holds in all ``real instances'' of Blocksworld.
  $\bar Q_{clear}$ is then sound for all of them,
  and the policy $\bar\pi_{clear}$ that solves the abstraction is a
  generalized plan for $\Q_{clear}$.
  (See the appendix for a complete derivation of the sufficient conditions.)
\end{example}

\section{Necessary Conditions and Invariants}

The conditions provided by $\G$ are only sufficient;
i.e., no conclusion about an abstraction $\bar Q$ for instance $P$
can be drawn if some reachable state $s$ in $P$ satisfies $\Pre(\bar a)$ but not $\Phi_{\bar a}$.
For reasoning in such cases, one needs \emph{necessary}
rather than sufficient conditions for soundness.

For obtaining necessary conditions, we need to assume
that the features in the abstraction are indeed monotone;
i.e., their value changes in a local manner along
the transitions in any instance $P$ in $\Q(\D)$.
For example, under monotonicity, $f$ decreases in 
$s\leadsto res(a(\bar o),s)$ iff
$C^{res(a(\bar o),s)}\subsetneq C^s$ iff
\begin{alignat*}{1}
  s \vDash \exists\bar x(\Psi_C(\bar x) \land \neg\Psi^a_C(\bar o,\bar x))
\end{alignat*}
(where $C$ is the concept that defines $f$) only if
\begin{alignat*}{1}
  s \vDash \exists\bar x(\Psi_C(\bar x) \land \neg S^a_C(\bar o,\bar x)) \,.
\end{alignat*}
We obtain necessary conditions similarly as before:
\Omit{
The reason by which $\G$ only provides sufficient conditions is
that it only considers changes on the concept's extensions by
subset containment (cf.\ Lemma~\ref{lemma:sufficient}).
However, we may assume that concepts change extension in
such a way, and then obtain necessary conditions:

\begin{definition}[Monotonicity]
  \label{def:monotoinicity}
  Let $\D$ be a domain
  and let $P$ be an instance in $\Q(\D)$.
  A concept $C$ for $\Q(\D)$ is \emph{monotone} for $P$ if for every \emph{reachable state}
  $s$ in $P$, and action $a$ that is applicable in $s$, either
  $C(P,s)\subseteq C(P,s')$, $C(P,s)\supseteq C(P,s')$, or
  $C(P,s)=C(P,s')$, where $s'{=}res(a,s)$.
  A first-order abstraction $\bar Q$ is monotone for $P$ if
  each feature $f$ in $\bar Q$ is given by a concept $C$ that
  is $\D$-definable and monotone. 
\end{definition}

}

\begin{theorem}[Necessary Conditions]
  Let $\D$ be a domain, let $\bar Q$ be a first-order abstraction for $\D$,
  and let $P$ be an instance in $\Q(\D)$ such that $\bar Q$ is \emph{sound and
  monotone} for $P$.
  If $s$ is a reachable state in $P$, then
  \[ \textstyle s\ \vDash\ \Pre(\bar a) \Rightarrow \exists\bar z\bigl(\bigvee_i \Gamma^{a_i}_{\bar a}(\bar z) \bigr) \]
  where $\Gamma^{a_i}_{\bar a}(\bar z)$ is the formula
  \[ \Pre(a_i(\bar z))\ \land
       \bigwedge_{\varphi\in\{inc,dec,eq,\top,\bot\}} 
         \ \bigwedge_{f_C \in \Delta^\varphi_{\bar a}}
           N^\varphi_C(\bar z) \,. \]
\end{theorem}

Necessary conditions are useful for showing that $\bar Q$ is
not sound for instance $P$: it is enough to find a reachable
state where the condition does not hold.
On the other hand, we cannot infer that $\Q$ is sound for $P$
only if the necessary condition does hold in $P$.

If the abstraction $\bar Q$ is sound for $\Q$, every instance $P$
in $\Q$ for which $\bar Q$ is monotone must satisfy the necessary
conditions.
Such conditions, that by definition hold in all reachable states,
are \emph{state invariants} in $P$. Therefore, sufficient and necessary
conditions for the soundness of an abstraction $\bar Q$ can be regarded
as candidates for invariants, which may then be verified by a theorem
prover in order to avoid an explicit enumeration of reachable states
\cite{hol4}.

\section{Other Examples}

In this section we present two additional examples. One for a
gripper problem with an arbitrary number of balls and grippers,
and the other about connectivity in directed graphs.

\subsection{Gripper}

We consider a domain $\D$ for Gripper with constants
$\constant{A}$ (destination) and $\constant{B}$ (origin)
for the rooms, objects $l$ and $r$ for the grippers, and objects $b_i$
for the different balls.
The predicates are $\at(r)$ and $\In(b,r)$
for the position of the robot and balls in rooms, $\carry(b,g)$
to indicate when ball $b$ is held by gripper $g$,
and $\free(g)$ to indicate that gripper $g$ is not holding any ball.
The action schemas are $a_1{\,=\,}\Move(r_1,r_2)$ for moving the robot,
and $a_2{\,=\,}\Pick(b,g,r)$ and $a_3{\,=\,}\Drop(b,g,r)$ for picking and
dropping balls in rooms using specific grippers.

\citeay{bonet:aaai2019} learn an abstraction 
that is made of a boolean feature $X$, 
and numerical features $B$, $C$, and $G$:
\begin{enumerate}[--]
  \item $X=\{ r : \at(r) \land r{=}\constant{A} \}$ tells whether the robot is in $\constant{A}$,
  \item $B=\{ b : \exists r(\In(b,r) \land r{\neq}\constant{A}) \}$ counts the balls in $\constant{B}$,
  \item $C=\{ b : \exists g(\carry(b,g)) \}$ counts the balls being held, and
  \item $G=\{ g : \free(g) \}$ counts the free grippers.
\end{enumerate}
The abstract actions in abstraction $\bar Q_{gripper}$ are:
\begin{enumerate}[--]
  \item $\text{pick} = \abst{\neg X,B>0,G>0}{\dec{B},\dec{G},\inc{C}}$,
  \item $\text{drop} = \abst{X,C>0}{\dec{C},\inc{G}}$,
  \item $\text{go1} = \abst{\neg X,B=0,C>0,G>0}{X}$,
  \item $\text{go2} = \abst{\neg X,C>0,G=0}{X}$,
  \item $\text{leave} = \abst{X,C=0,G>0}{\neg X}$.
\end{enumerate}
Both, go1 and go2, move the robot from $\constant{A}$ to $\constant{B}$.
Go1 moves the robot that still has room to pick more balls only when there
are no more balls to be picked at $\constant{B}$; go2 moves the robot when
it cannot hold any more balls.
\Omit{
The abstract actions partition the features according to
their effects on them:
\begin{enumerate}[--]
  \item pick: $\Delta^{inc}=\{C\}, \Delta^{dec}=\{B,G\}, \Delta^{eq}=\{X\}$,
  \item drop: $\Delta^{inc}=\{G\}, \Delta^{dec}=\{C\}, \Delta^{eq}=\{X,B\}$,
  \item go1:  $\Delta^{true}=\{X\}, \Delta^{eq}=\{B,C,G\}$,
  \item go2:  $\Delta^{true}=\{X\}, \Delta^{eq}=\{B,C,G\}$,
  \item leave: $\Delta^{false}=\{X\}, \Delta^{eq}=\{B,C,G\}$.
\end{enumerate}
}
The formulas $\Psi^{a_i}_{\bar a}(\bar z)$ 
are $\bot$ except for (conditions in $\Pre(a_i)$ removed to fit space):
\begin{alignat*}{1}
  \Psi^{a_2}_\text{pick}\  &=\       \forall x[\neg\carry(b,x)] \land r\neq \constant{A}\,, \\ 
  \Psi^{a_3}_\text{drop}\  &=\       \at(A) \land \neg\free(g) \land (r{=}\constant{A} \lor \exists x[\In(b,x) \land x{\neq}\constant{A}])\,, \\
  \Psi^{a_1}_\text{go1}\   &=\       \exists xy[\carry(x,y)] \land \exists x[\free(x)]\ \land \\
                           &\quad\ \ \ \forall xy[\In(x,y)\Rightarrow y=\constant{A}] \land r_1\neq\constant{A}\land r_2=\constant{A}\,, \\ 
  \Psi^{a_1}_\text{go2}\   &=\       \exists xy[\carry(x,y)] \land \forall x[\neg\free(x)] \land r_1 \neq \constant{A} \land r_2=\constant{A}\,, \\
  \Psi^{a_1}_\text{leave}\ &=\       \forall xy[\neg\carry(x,y)] \land \exists x[\free(x)] \land r_1 = \constant{A} \land r_2 \neq \constant{A} \,.
\end{alignat*}
For example, $\Psi^{a_1}_\text{pick}=\Psi^{a_3}_\text{pick}=\bot$ means
that the abstract pick action cannot be instantiated by any ground instance
of $\Move(r_1,r_2)$ or $\Drop(b,g,r)$: the first changes the feature $X$
that is not affected by pick, and the second increases $G$ in contradiction
with the effect $\dec{G}$.

On the other hand, $\Psi^{a_2}_\text{pick}=\forall x[\neg\carry(b,x)] \land r\neq \constant{A}$
means that the ground action $\Pick(b,g,r)$ instantiates the pick action
when $r\neq\constant{A}$, otherwise the effect $\dec{B}$ is not achieved,
and when the ball $b$ is not being held by any gripper $x$, otherwise
$\inc{C}$ is not met.
$\Psi^{a_2}_\text{pick}$ is logically implied at reachable states
by the preconditions of $\Pick(b,g,\constant{B})$ and the mutex
information that is polynomially computable.
Indeed, the preconditions are $\at(\constant{B})$,
$\In(b,\constant{B})$ and $\free(g)$, while the mutex invariants
include $\neg\In(b,r)\lor\neg\carry(b,g)$ for any $b$, $r$, and $g$.
Actually, we can show that the mutex information is enough to show
$\Pre(\bar a)\Rightarrow \exists\bar z(\bigwedge_i \Psi^{a_i}_{\bar a})$
for all the actions $\bar a$ in the abstraction.
Hence, the abstraction is sound for any instance of Gripper.

\subsection{Connectivity in Graphs}

We now consider a graph problem that involves the connectivity
of two designated vertices $\constant{s}$ and $\constant{t}$.
The domain $\D$ has constants $\constant{s}$ and $\constant{t}$,
a single binary predicate $E(x,y)$, and a single action schema
$\Link(x,y)$ that adds $E(x,y)$ and has no precondition.
The exact form of the initial situation or goal is not
relevant in the following discussion.

The abstraction $\bar Q$ defines two features: a boolean feature
$conn$ that is true iff $\constant{s}$ and $\constant{t}$ are
connected, and a numerical feature $n$ that
counts the total number of edges in the graph.
These features are defined with the concepts:
\begin{alignat*}{1}
  conn\ &=\ \{ (x,y) : E^*(x,y) \land x=\constant{s} \land y=\constant{t} \} \,, \\
  n\    &=\ \{ (x,y) : E(x,y) \} \,.
\end{alignat*}
There is a single abstract action $\bar a=\abst{}{\inc{n}}$.
Since $conn$ exists as a feature in $\bar Q$ and $\bar a$
does not affect it, no instantiation of $\bar a$ may modify
the $\constant{st}$-connectivity of the graph.
The guarantee $\Phi_{\bar a}$ for $\bar a$ is then
\begin{alignat*}{1}
  \exists z_1z_2 \bigl[ \neg E(z_1,z_2) \land
                        ( E^*(\constant{s},z_1) \land E^*(z_2,\constant{t}) \Rightarrow E^*(\constant{s},\constant{t}) ) \bigr] \,.
\end{alignat*}
That is, a sufficient condition for $\Link(z_1,z_2)$ to instantiate
$\bar a$, and thus increase $n$ and leave $conn$ intact, is that
there should be no edge between $z_1$ and $z_2$, and there should
be a path $\constant{s} \leadsto \constant{t}$ if there are
paths $\constant{s}\leadsto z_1$ and $z_2\leadsto\constant{t}$.
The first condition entails that the number of edges in the
graph indeed increases after applying $\Link(z_1,z_2)$, while
the second entails that the truth value of $conn$ does not
change with the application of the action: either it was true
and remains true, or it was false and remains false.

On other hand, the synthesis of the necessary condition
yields $\neg E(z_1,z_2)$ which is quite weak. The reason is
that the formula $\B_S(a,p^*)$ in Table~\ref{table:base} is not
strong enough. A better necessary condition is obtained when
the following term is added to the disjunction in $\B_S(a,p^*)$:
\begin{alignat*}{1}
  &\forall uv\bracket{\neg p(u,v)\notin\Post}\ \land \\
  &\qquad\qquad\exists uv(\bracket{p(u,v)\in\Post}\land p^*(x,u)\land p^*(v,y)) \,.
\end{alignat*}
This term says that the action removes no edge from the graph,
and adds one edge $(u,v)$ for existing paths $x\leadsto u$ and
$v\leadsto y$. Clearly, if this condition is met, the graph has
a path $x\leadsto y$ after the application of the action.
The resulting base thus remains valid for any domain $\D$ and,
in some cases, it provides tighter conditions.
Indeed, with this amendment, the necessary condition becomes
equal to the sufficient condition $\Phi_{\bar a}$.

Finally, observe that these conditions are not invariants.
The reason is that the abstraction is \emph{not} sound since
there are configurations (states) in which no edge can be
added to the graph without altering the $\constant{st}$-connectivity.
In such states, the abstract action $\bar a$ is still
applicable, as it does not have any precondition, but no
$\Link(z_1,z_2)$ action instantiates it.

\section{Discussion and Future Work}

Abstractions for generalized planning can be inductively obtained
from small samples of transitions from one or several instances of planning
problems. Although there are no guarantees on the soundness of the
abstractions, we have shown that the abstractions contain usable information
about the intended planning instances. Indeed, by analyzing the abstraction
with respect to the planning domain, we have shown how to obtain formulas
that capture some of the assumptions being made by the abstraction.
These assumptions can be either sufficient or necessary.
A sufficient condition is a guarantee
for generalization. Necessary conditions may be used
to show that an instance is not captured by the abstraction.

Sufficient conditions may also be used to improve the search for
instantiations of an abstract action $\bar a$ that is applicable in
a state $s$ in instance $P$. In the worst case, one needs to
iterate over every grounded action $a(\bar o)$ to find an instantiation
of $\bar a$. However, if $P$ satisfies the sufficient condition,
it is enough to find a tuple of objects $\bar o$ such that
$s\vDash\Phi_{\bar a}(\bar o)$, something that may be
easier to find than to iterate over the set of grounded actions.

The information provided by invariants has been exploited in
planning for different purposes and it is essential in some planning
paradigms.
The automatic synthesis of invariants is a computationally hard problem, 
and many of the existing techniques are based on a generate and test approach
that often yields an incomplete set of invariants
\cite{TIM_Fox_JAIR98,Gerevini1998InferringSC,Rintanen00,Rintanen08,concise-rep-malte-AI2009}.
We have shown how to use the synthesis of guarantees for abstractions
as a method for generating candidates for invariants.
Hence, the computation of abstractions for generalized planning may
be relevant even when the focus is to learn invariants rather than
to solve generalized planning problems.
As seen in the examples, some of these invariants are quite complex
and out-of-reach for state-of-the-art methods for invariant synthesis.

There are clear directions for future work.
First, come up with better bases for translation, based on
either $p^+$ or $p^*$, and understand better the strengths
and weaknesses of different bases.
Second, even though we are able to handle concepts that are
more general than those generated from concept grammars,
we are not yet able to accommodate distance features as defined
by \citeay{bonet:aaai2019}.
Finally, it may be the time to bring theorem provers into
the pipeline of generalized planning: from samples we obtain
abstractions $\bar Q$ by an inductive process, a solution
$\bar\pi$ for $\bar Q$ is computed with a FOND planner,
and an instance $P$ is then assured to be solved by $\bar\pi$
if $P$ satisfies the guarantee for $\bar Q$.
The latter may be automated with the help of theorem provers.

\section*{Acknowledgments}

B.\ Bonet is supported by a Banco Santander -- UC3M Chair of Excellence Award.
This work is partially funded by grants TIN2017-88476-C2-2-R and RTC-2017-6753-4
of the Spanish Ministerio de Econom\'{\i}a, Industria y Competitividad.

\bibliographystyle{named}
\bibliography{control}


\section*{Appendix: Guarantees for Blocksworld}

Derivation of sufficient guarantees for the soundness
of the abstraction $\bar Q_{clear}$ for Blocksworld.
The abstraction has a single numerical feature $n$ that
counts the number of blocks above $\constant{A}$:
$n=|\{x : \Psi(x) \}|$ where $\Psi(x)=\exists y(\on(x,y) \land \on^*(y,\constant{A}))\equiv\on^+(x,\constant{A})$.
There is also a unique abstract action that decrements
$n$ and has precondition $n>0$; i.e., $\bar a=\abst{n>0}{\dec{n}}$.

Blocksworld has been formulated in different ways.
We use a formulation with predicates $\on(x,y)$,
$\clear(x)$, and $\ontable(x)$, and with two action
schemas given by:
\begin{enumerate}[--]
  \item $\Newtower(z_1z_2)$: $\Pre=\{\on(z_1z_2),\clear(z_1)\}$, $\Eff=\{\neg\on(z_1z_2),\ontable(z_1),\clear(z_2)\}$, and
  \item $\Move(z_3z_4z_5)$: $\Pre=\{\on(z_3z_4),\clear(z_3),\clear(z_5)\}$,
    $\Eff=\{\neg\on(z_3z_4),\neg\clear(z_5),\on(z_3z_5),\clear(z_4)\}$.
\end{enumerate}
The sufficient condition for the decrement of the feature $n$ is
\begin{alignat*}{1}
  S^{dec}(\bar z) = \forall x(N^a_{\Psi}(\bar zx) \Rightarrow \Psi(x)) \land \exists x(\Psi(x) \land \neg N^a_{\Psi}(\bar zx))
\end{alignat*}
where $N^a_\Psi(\bar zx)$ is a necessary condition for
$\Psi(x)$ to hold after applying the action $a(\bar z)$
in the current state. In the following we derive the
formulas $\Psi^{a_i}_{\bar a}(\bar z)$ for each of the
two action schemas $a_i$ in the domain.
The guarantee for the soundness of the abstraction is
then $\{\Psi_{\bar a} \}$ where
$\Phi_{\bar a}=\exists\bar z(\bigvee_i \Psi^{a_i}_{\bar a}(\bar z))$.
By definition,
$\Psi^{a_i}_{\bar a}(\bar z)=\Pre(a(\bar z)) \land \Pre(\bar a)_C \land S^{dec}(\bar z)$
where $\Pre_C(\bar a)=\exists x(\on^+(x\constant{A}))$.

\bigskip\noindent\textbf{Action} $a_1 = \Newtower(z_1z_2)$: after simplification,
\begin{alignat*}{1}
  N^{a_1}_\Psi(\bar zx)\ &\equiv\ \exists y[\on(xy) \land on^*(y\constant{A}) \land xy\neq z_1z_2] \,, \\
  \neg N^{a_1}_\Psi(\bar zx)\ &\equiv\ \forall y[\on(xy) \land on^*(y\constant{A}) \Rightarrow xy=z_1z_2] \,.
\end{alignat*}
Then,
\begin{alignat*}{1}
  \forall x(N^{a_1}_\Psi \Rightarrow \Psi)\
    &\equiv\ \forall x[N^{a_1}_\Psi \Rightarrow \on^+(x\constant{A})]\ \equiv\ \top \,, \\
  \exists x(\Psi \land \neg N^{a_1}_\Psi)\
    &\equiv\ \on(z_1z_2) \land \on^*(z_2\constant{A}) \ \land \\
    &\quad\ \ \ \forall y [\on(z_1y) \land \on^*(y\constant{A}) \Rightarrow y=z_2] \,.
\end{alignat*}
Hence, $\Psi^{a_1}_{\bar a}(z_1z_2) \equiv \clear(z_1) \land \exists x ( \Psi(x) \land \neg N^{a_1}_\Psi(z_1z_2x))$.

\bigskip\noindent\textbf{Action} $a_2 = \Move(z_3z_4z_5)$: after simplification,
\begin{alignat*}{1}
  N^{a_2}_\Psi
    &\equiv [\on^*(xz_3) \land \on^*(z_5\constant{A})] \lor [\on^+(x\constant{A}) \land x\neq z_3]\ \lor \\
    &\quad \ \ \exists y[\on(xy) \land \on^*(y\constant{A}) \land y\neq z_4] \,, \\
  \neg N^{a_2}_\Psi
    &\equiv [\on^*(xz_3) \Rightarrow \neg\on^*(z_5\constant{A})]{\land}[\on^+(x\constant{A}) \Rightarrow x=z_3]{\land} \\
    &\quad \ \ \forall y[\on(xy) \land \on^*(y\constant{A}) \Rightarrow y=z_4] \,.
\end{alignat*}
Then,
\begin{alignat*}{1}
  \forall x(N^{a_2}_\Psi \Rightarrow \Psi)\
    &\equiv\ \forall x[ \on^*(xz_3) \land \on^*(z_5\constant{A}) \Rightarrow \on^+(x\constant{A}) ] \,, \\
  \exists x(\Psi \land \neg N^{a_2}_\Psi)\
    &\equiv\ \on^+(z_3\constant{A}) \land \neg\on^*(x_5\constant{A})\ \land \\
    &\quad\ \ \ \forall y[ \on(z_3y) \land \on^*(y\constant{A}) \Rightarrow y=z_4 ] \,.
\end{alignat*}
Hence, for $\bar z=z_3z_4z_5$, $\Psi^{a_2}_{\bar a}(\bar z)$ is logically equivalent to
\begin{alignat*}{1}
  \clear(z_3) \land \on(z_3z_4) \land \clear(z_5) \land \exists x ( \Psi(x){\land}\neg N^{a_2}_\Psi(\bar zx) ) .
\end{alignat*}


\end{document}